\documentclass[final]{cvpr}

\usepackage{times}
\usepackage{epsfig}
\usepackage{graphicx}
\usepackage{amsmath}
\usepackage{amssymb}

\usepackage{subfigure}
\usepackage{amssymb}
\usepackage{bm}
\usepackage{amsmath}
\usepackage{mathtools}
\usepackage{lineno}
\usepackage{multirow}
\usepackage{booktabs}
\usepackage{graphicx}

\usepackage{multirow} 
\usepackage{amsmath} 
\usepackage{xcolor}

\usepackage[linesnumbered,boxed,ruled,commentsnumbered]{algorithm2e}
\SetKwInOut{KIN}{KwIn}

\def\onedot{.}
\def\eg{\emph{e.g}\onedot} 
\def\ie{\emph{i.e}\onedot} 

\def\Vec#1{{\boldsymbol{#1}}}
\def\Mat#1{{\boldsymbol{#1}}}
\newcommand\dlmu[2][4cm]{\hskip1pt\underline{\hb@xt@ #1{\hss#2\hss}}\hskip3pt}

\newtheorem{definition}{Definition}

\newtheorem{proposition}{Proposition}

\newenvironment{proof}{\noindent\it Proof.\quad}{\hfill $\square$\par}

\usepackage[pagebackref=true,breaklinks=true,colorlinks,bookmarks=false]{hyperref}

\def\cvprPaperID{3302} 
\def\confYear{CVPR 2021}

\begin{document}
\title{A Hyperbolic-to-Hyperbolic Graph Convolutional Network}

\author{Jindou Dai, Yuwei Wu\thanks{Corresponding author}~, Zhi Gao, and Yunde Jia\\
Laboratory of Intelligent Information Technology, School of Computer Science, \\ Beijing Institute of Technology (BIT), Beijing, 100081, China.\\
{\tt\small \{daijindou, wuyuwei, gaozhi\_2017, jiayunde\}@bit.edu.cn}
}

\maketitle
\pagestyle{empty}  
\thispagestyle{empty}

\begin{abstract}
Hyperbolic graph convolutional networks (GCNs) demonstrate powerful representation ability to model graphs with hierarchical structure. 
Existing hyperbolic GCNs resort to tangent spaces to realize graph convolution on hyperbolic manifolds, which is inferior because tangent space is only a local approximation of a manifold. 
In this paper, we propose a hyperbolic-to-hyperbolic graph convolutional network (H2H-GCN) that directly works on hyperbolic manifolds. 
Specifically, we developed a manifold-preserving graph convolution that consists of a hyperbolic feature transformation and a hyperbolic neighborhood aggregation. 
The hyperbolic feature transformation works as linear transformation on hyperbolic manifolds. 
It ensures the transformed node representations still lie on the hyperbolic manifold by imposing the orthogonal constraint on the transformation sub-matrix. 
The hyperbolic neighborhood aggregation updates each node representation via the Einstein midpoint. 
The H2H-GCN avoids the distortion caused by tangent space approximations and keeps the global hyperbolic structure. 
Extensive experiments show that the H2H-GCN achieves substantial improvements on the link prediction, node classification, and graph classification tasks. 

\end{abstract}

\begin{figure}[h]
\tiny
\centering

\subfigure[Existing hyperbolic GCNs.]{\begin{minipage}[t]{\linewidth}
{
\includegraphics[width=1.0\linewidth]{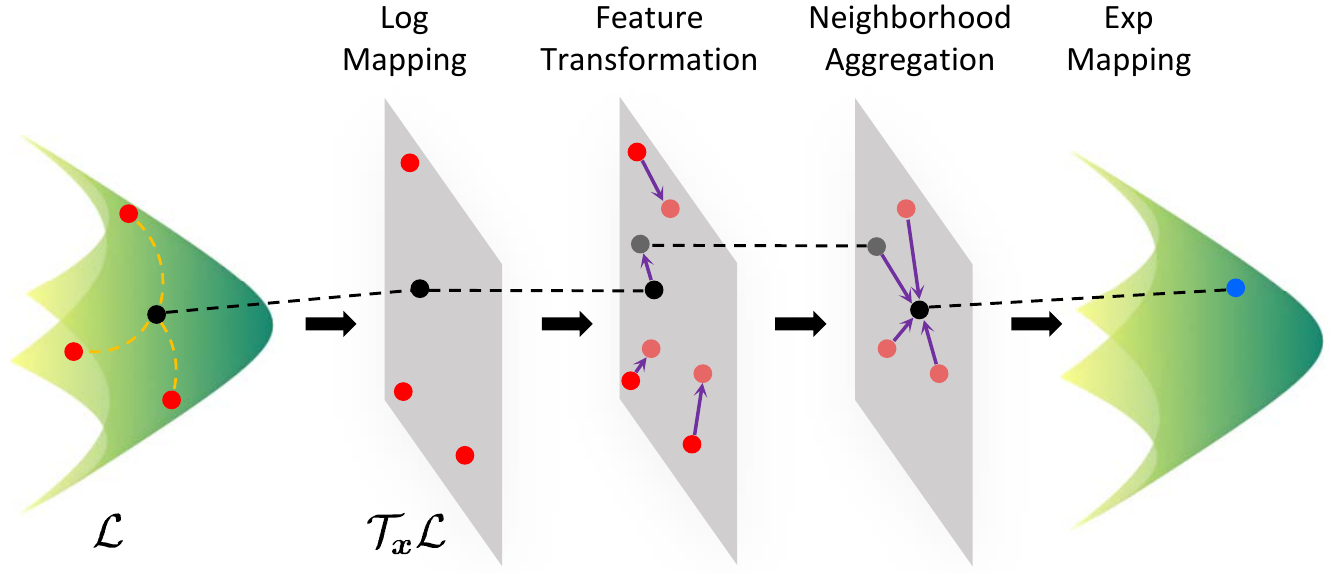}
}
\label{figure:tangentHGCN}
\end{minipage}}
\subfigure[The proposed H2H-GCN.]{\begin{minipage}[b]{\linewidth}
{
\includegraphics[width=1.0\linewidth]{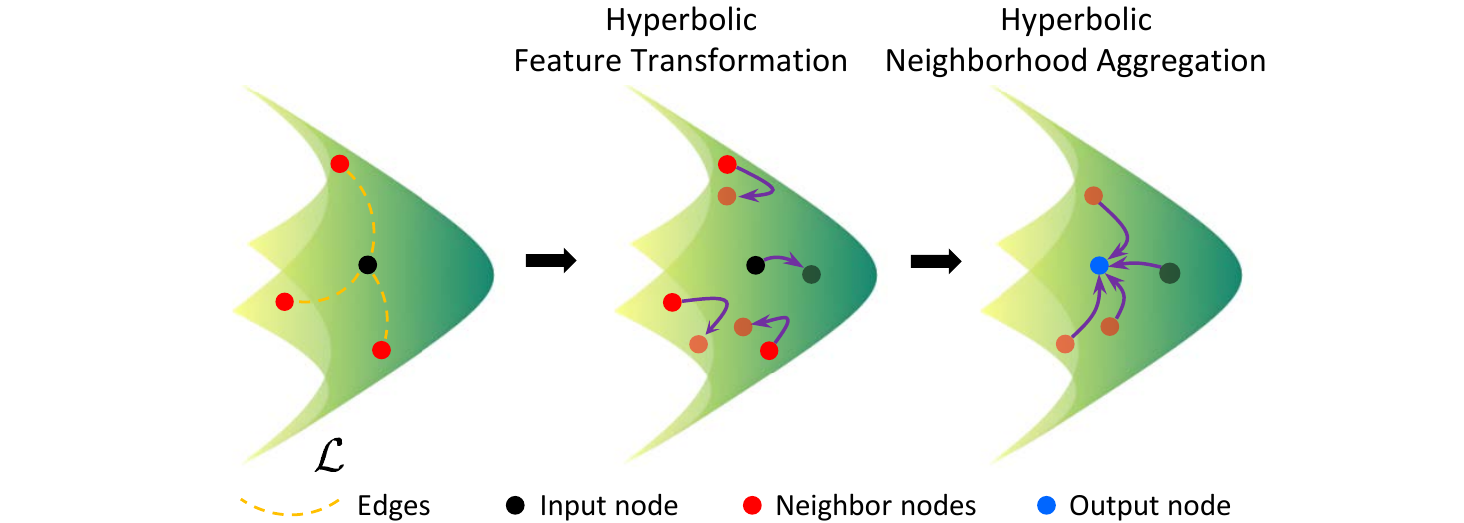}
}
\label{figure:H2H-GCN}
\end{minipage}}
\caption{
Comparisons of exisiting hyperbolic GCNs and the proposed H2H-GCN. 
At the $\ell$-th layer, (a) existing hyperbolic GCNs performs Euclidean graph convolutional opeartions, \eg, feature transformation and neighborhood aggregation, in the tangent space $\mathcal{T}_{\Vec{x}}\mathcal{L}$ that is a local approximation of the hyperbolic manifold $\mathcal{L}$; 
(b) H2H-GCN directly performs a hyperbolic feature transformation and a hyperbolic neighborhood aggregation on the hyperbolic manifold to learn node representations, keeping the global hyperbolic structure. 
}
\label{fig:experimentvisual}
\vspace{-8pt}
\end{figure}

\section{Introduction}
\label{section:introduction}

Graph convolutional networks (GCNs) have attracted increasing attention for graph representation learning, where nodes in a graph are typically embedded into Euclidean spaces~\cite{kipf2016semi,zhang2018end,xu2018powerful,velickovic2019deep,gilmer2017neural,hamilton2017inductive}. 
Several works reveal that many graphs, such as social networks and citation networks, exhibit a highly hierarchical structure~\cite{clauset2008hierarchical,krioukov2010hyperbolic,papadopoulos2012popularity}. 
Recent studies have shown that hyperbolic spaces can well capture such hierarchical structure compared to Euclidean spaces~\cite{nickel2017poincare,nickel2018learning,de2018representation}. 
Different from Euclidean spaces with zero curvature, hyperbolic spaces possess a constant negative curvature, which allows for an exponential growth of space volume with radius. 
This property of hyperbolic spaces pretty meets the requirements of hierarchical data (\eg,~trees) that need an exponential amount of space for branching, and encourages the development of GCNs in hyperbolic spaces to capture the hierarchical structure underlying graphs. 

Existing hyperbolic GCNs~\cite{liu2019hyperbolic,chami2019hyperbolic,zhang2019hyperbolic} resort to tangent spaces to realize graph convolution in hyperbolic spaces. 
Since the hyperbolic space is a Riemannian manifold rather than a vector space, basic operations (such as matrix-vector multiplication and vector addition) well defined in Euclidean spaces are not applicable in hyperbolic space. 
To generalize graph convolution to the hyperbolic space, the works in \cite{liu2019hyperbolic,chami2019hyperbolic,zhang2019hyperbolic} first flatten a hyperbolic manifold, and then apply Euclidean graph convolutional operations in the tangent space. 
The results are projected back to the hyperbolic manifold. 
The procedures follow a manifold-tangent-manifold scheme, as shown in Figure~\ref{figure:tangentHGCN}. 
These methods has promoted the development of GCNs in hyperbolic spaces and achieved good performance. 
However, the mapping between the manifold and the tangent space is only locally diffeomorphic, which may distort the global structure of the hyperbolic manifold, especially frequently using tangent space approximations \cite{huang2017cross,tuzel2008pedestrian}.

In this paper, we propose to design a hyperbolic GCN that directly works on the hyperbolic manifold to keep global hyperbolic structure, rather than relying on the tangent space. 
This requires that each step of graph convolution, \eg, feature transformation and neighborhood aggregation, satisfies a manifold-to-manifold principle. 
To this end, we present a hyperbolic-to-hyperbolic graph convolutional network (H2H-GCN), where graph convolutional operations are directly conducted on the hyperbolic manifold. 

Specifically, we developed a manifold-preserving graph convolution consisting of a hyperbolic feature transformation and a hyperbolic neighborhood aggregation. 
The hyperbolic feature transformation plays the role of linear transformation on hyperbolic manifolds, which requires multiplication of node representations by a transformation matrix. 
We constrain the transformation matrix to be a block diagonal matrix composed of a scalar 1 and an orthogonal matrix to ensure the transformed node representations still reside on the hyperbolic manifold. 
For hyperbolic neighborhood aggregation, we adopt the Einstein midpoint as the weighted message of neighbor nodes to update a node representation. 
Figure~\ref{figure:H2H-GCN} depicts that H2H-GCN directly carries out the two steps on hyperbolic manifolds. 
In constrast to existing hyperbolic GCNs, the proposed H2H-GCN can avoid the distortion caused by tangent space approximations and keep the global hyperbolic structure underlying graphs. 
We summarize the contributions of this paper as follows. 
\begin{itemize}
\item{} We propose a hyperbolic-to-hyperbolic graph convolutional network that directly performs graph convolution on hyperbolic manifolds, keeping the global hyperbolic structure underlying graphs. 
To the best of our knowledge, this is the first hyperbolic GCN without relying on tangent spaces. 
\item{} We developed a hyperbolic feature transformation that is a linear transformation on hyperbolic manifolds. 
The manifold constraint on the transformed hyperbolic representations is ensured by imposing the orthogonal constraint on the transformation sub-matrix. 
\end{itemize}

\section{Related Work}
\label{section:related_work}
GCNs generalize classical convolutional neural networks to graph domains. 
To realize the convolution on graphs, there are two types of GCNs. 
Spectral-based GCNs \cite{bruna2013spectral,defferrard2016convolutional,kipf2016semi,hu2020going} are based on the convolution theorem to perform convolution by transforming graph signals into the spectral domain via the graph Fourier transform. 
Spatial-based GCNs \cite{gilmer2017neural,hamilton2017inductive,velivckovic2017graph,pei2020geom,velickovic2019deep,xu2018powerful} update node representations by aggregating the message from its neighbor nodes, just like applying convolutional kernel on a local image patch. 
Despite a solid theoretical foundation of spectral-based GCNs, spatial-based GCNs have shown more superiorities due to efficiency, generality and flexibility.

Researchers discoverd that many graphs, \eg,~social networks and biological networks, usually exhibit a highly hierarchical structure \cite{krioukov2010hyperbolic,papadopoulos2012popularity}. 
Krioukov \etal~\cite{krioukov2010hyperbolic} pointed that the properties of strong clustering and power-law degree distribution in such graphs can be explained as a hidden hierarchy. 
Recent works have demonstrated powerful representation ability of hyperbolic spaces to model hierarchies that underlie taxonomies \cite{nickel2017poincare,nickel2018learning}, knowledge graphs \cite{sun2020knowledge,balazevic2019multi}, images \cite{khrulkov2020hyperbolic}, semantic classes \cite{liu2020hyperbolic}, actions \cite{long2020searching}, \textit{etc} \cite{sonthalia2020tree,chami2020trees,weber2020robust}, achieving promising performance. 
Liu \etal~\cite{liu2019hyperbolic} and Chami \etal~\cite{chami2019hyperbolic} proposed hyperbolic GCNs that extend GCNs to hyperbolic spaces to capture the hierarchy underlying graphs. 
The main difference with our work is that they perform Euclidean graph convolutional operations in the tangent space, following a manifold-tangent-manifold scheme. 
The proposed H2H-GCN developed a hyperbolic graph convolution in the hyperbolic space without relying on tangent spaces. 
We designed a Lorentz linear transformation for feature transformation on the Lorentz model, and adopted Einstein midpoint to calculate manifold statistics \cite{fletcher2013geodesic} as aggregation function. 
We claim that such a manifold-to-manifold learning principle can avoid the distortion caused by tangent space approximations and keep the global hyperbolic structure, that is beneficial to graph representation learning. 
\vspace{-2pt}

\section{Preliminaries}
\label{section:background}

\subsection{Hyperbolic Spaces}
\label{subsection:hyperbolic_geometry}

A Riemannian manifold $(\mathcal{M},g)$ is a differentiable manifold $\mathcal{M}$ equipped with a metric tensor $g$. 
It can be locally approximated to a linear Euclidean space at an arbitrary point $\Vec{x} \in \mathcal{M}$, and the approximated space is termed as a tangent space $\mathcal{T}_{\Vec{x}}\mathcal{M}$. 
Hyperbolic spaces are smooth Riemannian manifolds with a constant negative curvature \cite{benedetti2012lectures}. 
There are five isometric models of hyperbolic spaces: the Lorentz model (a.k.a the hyperbolid model), the Klein model, the Jemisphere model, the Poincar$\mathrm{\acute{e}}$ ball model, and the Poincar$\mathrm{\acute{e}}$ half-space model \cite{cannon1997hyperbolic}. 
In this paper, we choose the Lorentz model due to its numerical stablity \cite{nickel2018learning}.

Formally, the Lorentz model of an $n$-dimensional hyperbolic space is defined by the manifold $\mathcal{L}=\{\Vec{x} = [x_{0}, x_{1}, \cdots, x_{n}] \in\mathbb{R}^{n+1}: \langle\Vec{x},\Vec{x}\rangle_{\mathcal{L}}=-1,\Vec{x}_{0}>0\}$ endowed with the metric tensor $g=\mathrm{diag}([-1,\Vec{1}_{n}^{\top}])$ where $\mathrm{diag}(\cdot)$ function transforms a vector to a diagonal matrix. 
The Lorentz inner product induced by $g$ is defined as 
\begin{small}
\begin{equation}
\label{equation:Lorentz_inner_product}
\langle\Vec{x},\Vec{y}\rangle_{\mathcal{L}} = \Vec{x}^{\top}g\Vec{y}= - {x}_{0}{y}_{0}+\sum_{i=1}^{n}{x}_{i}{y}_{i}. 
\end{equation}
\end{small}

\vspace{-10pt}
\noindent In the following, we describe necessary operations. 

\noindent \textbf{Distance.} The distance on a manifold is termed as a geodesic that is commonly a curve representing the shortest path between two nodes. 
For $\forall \Vec{x},\Vec{y} \in \mathcal{L}$, the distance between them is given by 
\begin{equation}
\label{equation:distance}
d_{\mathcal{L}}(\Vec{x}, \Vec{y}) = \mathrm{arcosh}(-\langle\Vec{x},\Vec{y}\rangle_{\mathcal{L}}), 
\vspace{-2pt}
\end{equation}
where $\mathrm{arcosh}(\cdot)$ is the inverse hyperbolic cosine function. 

\vspace{2pt}
\noindent\textbf{Exponential and logarithmic maps.}
An exponential map $\mathrm{exp}_{\Vec{x}}(\Vec{v})$ is the function projecting a tangent vector $\Vec{v} \in \mathcal{T}_{\Vec{x}}\mathcal{M}$ onto $\mathcal{M}$. 
A logarithmic map projects vectors on the manifold back to the tangent space satisfying $\mathrm{log}_{\Vec{x}}(\mathrm{exp}_{\Vec{x}}(\Vec{v}))=\Vec{v}$. 
For $\Vec{x}, \Vec{y} \in \mathcal{L}$, and $\Vec{v} \in \mathcal{T}_{\Vec{x}}\mathcal{L}$ the exponential map $\mathrm{exp}_{\Vec{x}}:\mathcal{T}_{\Vec{x}}\mathcal{L} \to \mathcal{L}$ and the logarithmic map $\mathrm{log}_{\Vec{x}}:\mathcal{L} \to \mathcal{T}_{\Vec{x}}\mathcal{L}$ are given by 
\begin{small}
\begin{equation}
\label{equation:exp}
\begin{aligned}
\mathrm{exp}_{\Vec{x}}(\Vec{v}) 
&= \mathrm{cosh}(\|\Vec{v}\|_{\mathcal{L}})\Vec{x}+\mathrm{sinh}(\|\Vec{v}\|_{\mathcal{L}})\frac{\Vec{v}}{\|\Vec{v}\|_{\mathcal{L}}}, \\
\end{aligned}
\end{equation}
\begin{equation}
\label{equation:log}
\begin{aligned}
\mathrm{log}_{\Vec{x}}(\Vec{y}) 
&= \frac{\mathrm{arcosh}(-\langle\Vec{x},\Vec{y} \rangle_{\mathcal{L}})}{\sqrt{\langle \Vec{x}, \Vec{y} \rangle_{\mathcal{L}}^{2}-1}}(\Vec{y}+\langle \Vec{x} , \Vec{y}\rangle_{\mathcal{L}}\Vec{x}), \\
\end{aligned}
\end{equation}
\end{small}
where $\|\Vec{v}\|_{\mathcal{L}}=\sqrt{\langle \Vec{v}, \Vec{v} \rangle_{\mathcal{L}}}$ is the norm of $\Vec{v}$. 

\vspace{2pt}
\noindent\textbf{Isometric isomorphism.}
The Poincar$\mathrm{\acute{e}}$ ball model $\mathcal{B}$ and the Klein model $\mathcal{K}$ are two other models of hyperbolic spaces. 
The bijections between a node $\Vec{x}=[{x}_{0},{x}_{1},\cdots,{x}_{n}] \in \mathcal{L}$ and its unique corresponding node $\Vec{b}=[b_{0},b_{1}, \cdots,b_{n-1}] \in \mathcal{B}$ are given by 
\begin{small}
\begin{equation}
\label{equation:bijection_LB}
\begin{aligned}
p_{\mathcal{L} \to \mathcal{B}}(\Vec{x}) = \frac{[{x}_{1},\cdots,{x}_{n}]}{{x}_{0}+1}, \
p_{\mathcal{B} \to \mathcal{L}}(\Vec{b}) = \frac{[1+\|\Vec{b}\|^{2},2\Vec{b}]}{1-\|\Vec{b}\|^{2}}. 
\end{aligned}
\end{equation}
\end{small}

\vspace{-10pt}
\noindent The bijections between $\Vec{x}=[{x}_{0},{x}_{1},\cdots,{x}_{n}] \in \mathcal{L}$ and its unique corresponding node $\Vec{k} = [k_{0},k_{1}, \cdots,k_{n-1}]  \in \mathcal{K}$ are given by 
\begin{small}
\begin{equation}
\label{equation:bijection_LK}
\begin{aligned}
p_{\mathcal{L} \to \mathcal{K}}(\Vec{x}) = \frac{[{x}_{1},\cdots,{x}_{n}]}{{x}_{0}}, \
p_{\mathcal{K} \to \mathcal{L}}(\Vec{k}) = \frac{1}{\sqrt{1-\|\Vec{k}\|^{2}}}[1,\Vec{k}]. 
\end{aligned}
\end{equation}
\end{small}

\vspace{-10pt}
\noindent Geometric relationships among the Lorentz model $\mathcal{L}$, the Poincar$\mathrm{\acute{e}}$ ball model $\mathcal{B}$ and the Klein model $\mathcal{K}$ are presented in Figure~\ref{fig:relation_LPK}.

\subsection{Graph Convolutional Networks}
\label{subsection:GCNs}
\begin{figure}[t]
\centering
\includegraphics[width=0.75\columnwidth]{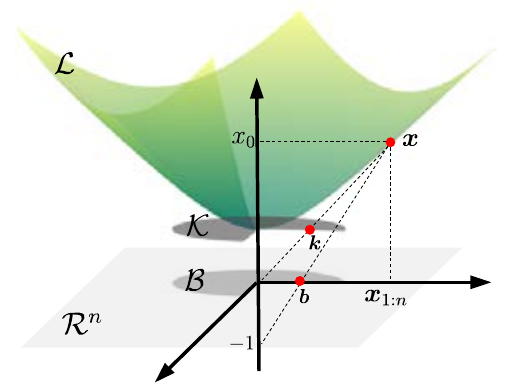}
\caption[width=2.2cm]{Geometric relationships among $\mathcal{L}$, $\mathcal{B}$ and $\mathcal{K}$. }
\label{fig:relation_LPK}
\vspace{-5pt}
\end{figure}

Let $\mathcal{G}=(\mathcal{V}, \mathcal{E})$ be a graph with a vertex set $\mathcal{V}$ and an edge set $\mathcal{E}$, 
and $ \{\Vec{x}_{i}^{E}  \}_{i \in \mathcal{V}}$ be node features where $^{E}$ denotes Euclidean representations. 
At the $\ell$-th layer of GCNs, the graph convolution can be formulated into two steps. 

\noindent\textbf{Feature Transformation:}
\begin{equation}
\label{equation:FT_in_E}
\Vec{\bar{h}}_{i}^{\ell,E} = \Mat{W}^{\ell}\Vec{h}_{i}^{\ell-1,E}. 
\end{equation}
where $\Vec{h}_{i}^{\ell-1,E}$ denotes the $i$-th node's representation at the ($\ell$-$1$)-th layer and $\Vec{h}_{i}^{0,E} = \Vec{x}_{i}^{E}$. 
$\Mat{W}^{\ell}$ is the learnable transformation matrix at the $\ell$-th layer. 
$\Vec{\bar{h}}_{i}^{\ell,E}$ is the intermediate representation of the $i$-th node, ready for the next step. 

\noindent\textbf{Neighborhood Aggregation:}
\begin{equation}
\label{equation:NAgg_in_E}
\left\{
   \begin{aligned}
      \Vec{m}_{i}^{\ell,E} & = (\Vec{\bar{h}}_{i}^{\ell,E} + \sum_{j \in \mathcal{N}(i)}w_{ij}\Vec{\bar{h}}_{j}^{\ell,E}) \\
      \Vec{h}_{i}^{\ell,E} & = \sigma(\Vec{m}_{i}^{\ell,E}) \\
   \end{aligned}
\right., 
\end{equation}
where $\mathcal{N}(i)$ denotes the set of neighbor nodes of the $i$-th node, and $w_{ij}$ is the aggregation weight. %
$\Vec{m}_{i}^{\ell,E}$ is the aggregated message, that is sent to a non-linear activation function $\sigma(\cdot)$ to output the node representation $\Vec{h}_{i}^{\ell,E}$ at the $\ell$-th layer. 

By stacking multiple graph convolutional layers, the feature transformation enables GCNs to learn desirable node embeddings for a target task, \eg,~more discriminative for classifications. 
The neighborhood aggregation enables GCNs to exploit graph topology structures.

\section{Hyperbolic-to-Hyperbolic GCN}
\label{section:h2hGCN}
We present H2H-GCN that directly performs graph convolution on hyperbolic manifolds to keep global hyperbolic structure. 
First, we explain how to generate hyperbolic node representations as input node features are usually Euclidean. 
Then, we elaborate the developed hyperbolic feature transformation and hyperbolic neighborhood aggregation. 
Next, we construct the H2H-GCN architecture used for link prediction, node classification and graph classification. 
Finally, we describe how to optimize parameters in the H2H-GCN.

\subsection{Hyperbolic Node Representations}
\label{subsection:generate_hyperbolic}

Let $\{\Vec{x}_{i}^{E}  \}_{i \in \mathcal{V}}$ be input Euclidean node features, and $\Vec{o} \coloneqq [1,0,\cdots,0] $ denote the origin on the manifold $\mathcal{L}$ of the Lorentz model. 
There is $\langle \Vec{o}, [0,\Vec{x}_{i}^{E}] \rangle_{\mathcal{L}} = 0$, where $\langle\cdot,\cdot \rangle_{\mathcal{L}}$ denotes the Lorentz inner product defined in Eq.~\eqref{equation:Lorentz_inner_product}. 
We can reasonably regard $[ 0,\Vec{x}_{i}^{E} ]$ as a node on the tangent space at the origin $\Vec{o}$. 
H2H-GCN uses the exponential map defined in Eq.~\eqref{equation:exp} to generate hyperbolic node representations on the Lorentz model: 
\begin{equation}
\label{equation:obtain_hyper}
\begin{aligned}
\Vec{x}_{i}^{\mathcal{L}} &= \mathrm{exp}_{\Vec{o}}\big([0, \Vec{x}_{i}^{E}]\big) \\
&= \Big[ \mathrm{cosh}\big(\| \Vec{x}_{i}^{E} \|_{2} \big), \mathrm{sinh} \big(\| \Vec{x}_{i}^{E} \|_{2}\big)\frac{\Vec{x}_{i}^{E}}{\| \Vec{x}_{i}^{E} \|_{2}} \Big]. 
\end{aligned}
\end{equation}

\vspace{-10pt}
\subsection{Hyperbolic Feature Transformation}
\label{subsection:feature_transformation}
The feature transformation in (Euclidean) GCNs defined in Eq.~\eqref{equation:FT_in_E} is a linear transformation realized via a matrix-vector multiplication. 
Nevertheless, it will break the hyperbolic manifold constraint while applying matrix-vector multiplication to hyperbolic node representations, making the transformed nodes not lie on hyperbolic manifolds. 
We developed a Lorentz linear transformation to tackle this problem. 
\vspace{-5pt}
\begin{definition} 
\label{definition:FA_in_L}
(The Lorentz linear transformation). 
For any $\Vec{x} \in \mathcal{L}$, the Lorentz linear transformation is defined as 
\begin{equation}
\label{equation:FA_in_L}
\begin{aligned}
    \Vec{y} &= \Mat{W} \Vec{x} \\
    \mathrm{s.t.} ~  \Mat{W} = &
    \begin{bmatrix} 
    1 & \Vec{0}^{\top} \\
    \Vec{0} & \widehat{\Mat{W}}
    \end{bmatrix}
    , \widehat{\Mat{W}}^{\top} \widehat{\Mat{W}} = \Mat{I}, 
\end{aligned}
\vspace{-8pt}
\end{equation}
where $\Mat{W}$ is a transformation matrix, and $\widehat{\Mat{W}}$ is called a transformation sub-matrix. 
$\Vec{0}$ is a column vector of zeros, and $\Mat{I}$ is an identity matrix. 
\end{definition}
\begin{proposition}
\label{lemma:Lorentz_linear_transformation}
The Lorentz linear transformation defined in Definition.~\ref{definition:FA_in_L} is manifold-preserving. 
It ensures that the output $\Vec{y}$ still lies on the manifold $\mathcal{L}$ of the Lorentz model. 
\end{proposition}

\begin{proof} 
For any $\Vec{x} = [{x}_{0}, {x}_{1}, \cdots, {x}_{n}] \in \mathcal{L}$, we have 
\begin{equation*}
\begin{aligned}
    - {x}_{0}^{2} + \Vec{x}_{1:n}^{\top} \Vec{x}_{1:n} = -1, ~\textrm{and } {x}_{0} > 0, \\
\end{aligned}
\end{equation*}
where $\Vec{x}_{1:n} = [{x}_{1}, {x}_{2}, \cdots, {x}_{n}]$. 
After applying the Lorentz linear transformation to $\Vec{x}$, we have 
\begin{equation*}
\Vec{y} = \Mat{W}\Vec{x} = \Big[{x}_{0}, \widehat{\Mat{W}}\Vec{x}_{1:n} \Big], 
\end{equation*}
satisfying 
\begin{equation*}
 y_{0} = x_{0} > 0,\\
\end{equation*}
and 
\begin{equation*}
\begin{aligned}
    & \langle \Vec{y}, \Vec{y} \rangle_{\mathcal{L}}  = - {x}_{0}^{2} + (\widehat{\Mat{W}}\Vec{x}_{1:n})^{\top} \widehat{\Mat{W}}\Vec{x}_{1:n}\\
    & = - {x}_{0}^{2} + \Vec{x}_{1:n}^{\top} (\widehat{\Mat{W}}^{\top} \widehat{\Mat{W}}) \Vec{x}_{1:n}\\
    & = - {x}_{0}^{2} + \Vec{x}_{1:n}^{\top} \Vec{x}_{1:n} \\
    & = - 1. \\
\end{aligned}
\end{equation*}
Thus, $\Vec{y}$ lies on the manifold $\mathcal{L}$ of the Lorentz model. 
\end{proof}

We utilize the Lorentz linear transformation as the hyperbolic feature transformation in H2H-GCN. 
At the $\ell$-th layer, we take the node representation from the previous layer $\Vec{h}_{i}^{\ell-1, \mathcal{L}}$ and the transformation matrix $\Mat{W}^{\ell}$ as input. 
The $i$-th node's intermediate representation is calculated by 
\begin{equation}
\label{equation:FA_in_L2}
\begin{aligned}
    \Vec{\bar{h}}_{i}^{\ell, \mathcal{L}} &= \Mat{W}^{\ell} \Vec{h}_{i}^{\ell-1, \mathcal{L}} \\
        \mathrm{s.t.}~  \Mat{W}^{\ell} &= 
        \begin{bmatrix}
        1 & \Vec{0}^{\top} \\
        \Vec{0} & \widehat{\Mat{W}}^{\ell}
        \end{bmatrix}
        , \widehat{\Mat{W}}^{\ell^{\top}} \widehat{\Mat{W}} = \Mat{I}, 
\end{aligned}
\end{equation}
where $\Vec{h}_{i}^{0, \mathcal{L}} = \Vec{x}_{i}^{\mathcal{L}}$. 
The intermediate node representation $\Vec{\bar{h}}_{i}^{\ell, \mathcal{L}}$ is ready for hyperbolic neighborhood aggregation in Section~\ref{subsection:neighborhood_aggregation}. 
We describe an effective way to learn transformation matrix $\Mat{W}^{\ell}$, a constrained parameter, via optimization on a Stiefel manifold in Section~\ref{subsection:optimization}. 
\begin{figure}[t]
\centering
\includegraphics[width=0.38\textwidth]{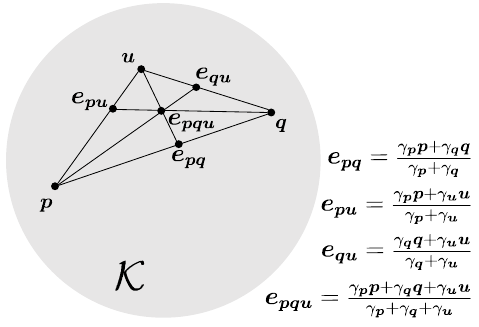}
\caption{Einstein midpoint on the Klein model $\mathcal{K}$, taking three nodes $\Vec{u}, \Vec{p}, \Vec{q} \in \mathcal{K}$ for example. 
}
\label{fig:Einstein}
\vspace{-10pt}
\end{figure}

\subsection{Hyperbolic Neighborhood Aggregation}
\label{subsection:neighborhood_aggregation}
The neighborhood aggregation in GCNs defined in Eq.\eqref{equation:NAgg_in_E} updates a node representation by aggregating the message from its neighbor node set, enabling GCNs to capture graph topological structure. 
A generalization of Euclidean mean aggregation in hyperbolic spaces is Fr\'echet mean \cite{frechet1948elements}. 
However, Fr\'echet mean is difficult to apply because it does not have a closed-form solution. 
We adopt the Einstein midpoint \cite{ungar2005analytic} as the hyperbolic neighborhood aggregation in H2H-GCN. 
In this case, our hyperbolic neighborhood aggregation possesses two desirable properties: translation invariance and rotation invariance. 
The aggregated hyperbolic average is invariant to translating the input node set by a same distance in a common direction, and invariant to rotating the input node set by a same angle around the origin. 

The Einstein midpoint takes the form in the Klein model, illustrated in Figure~\ref{fig:Einstein}. 
We first project the intermediate node representations from the Lorentz model to the Klein model, and then calculate the hyperbolic average via the Einstein midpoint. 
The aggregated hyperbolic average on the Klein model is projected back to the Lorentz model. 
Formally, given the intermediate representation of a node $\Vec{\bar{h}}_{i}^{\ell,\mathcal{L}}$ and the intermediate representations of its neighbor nodes $\{\Vec{\bar{h}}_{j}^{\ell, \mathcal{L}}\}_{j \in \mathcal{N}(i)}$, the hyperbolic neighborhood aggregation on the Lorentz model is given by 
\begin{equation}
\label{equation:einstein_midpoint}
\left\{
   \begin{aligned}
      \Vec{\bar{h}}_{j}^{\ell, \mathcal{K}} & = p_{\mathcal{L} \to \mathcal{K}}(\Vec{\bar{h}}_{j}^{\ell, \mathcal{L}}) \\
      \Vec{m}_{i}^{\ell, \mathcal{K}} & = \sum_{j \in \widehat{\mathcal{N}}(i)}\gamma_{j}\Vec{\bar{h}}_{j}^{\ell, \mathcal{K}} / \sum_{j \in \widehat{\mathcal{N}}(i)}\gamma_{j} \\
      \Vec{m}_{i}^{\ell, \mathcal{L}} & = p_{\mathcal{K} \to \mathcal{L}}(\Vec{m}_{i}^{\ell, \mathcal{K}}) \\
   \end{aligned}
\right., 
\end{equation}
where $\Vec{\bar{h}}_{j}^{\ell, \mathcal{K}}$ is the $j$-th node's intermediate representation on the Klein model. 
$p_{\mathcal{L} \to \mathcal{K}}(\cdot)$ and $p_{\mathcal{K} \to \mathcal{L}}(\cdot)$ are the isometric and isomorphic bijections between the Lorentz model and the Klein model as defined in Eq.~\eqref{equation:bijection_LK}. 
$\widehat{\mathcal{N}}(i)$ is a node set consisting of the $i$-th node and its neighbor nodes. 
$\gamma_{j}=\frac{1}{\sqrt{1-\| \bar{\Vec{h}}_{j}^{\mathcal{K}} \|^{2}}}$ denotes the Lorentz factor. 
$\Vec{m}_{i}^{\ell, \mathcal{K}}$ is the hyperbolic average on the Klein model that aggregates the message from $\widehat{\mathcal{N}}(i)$ via the Einstein midpoint. 
We get the hyperbolic average on the Lorentz model $\Vec{m}_{i}^{\ell, \mathcal{L}}$ by projecting $\Vec{m}_{i}^{\ell, \mathcal{K}}$ to $\mathcal{L}$.

The non-linear activation plays an important role in GCNs, which prevents a multi-layer network from collapsing into a single layer network. 
However, applying commonly-used non-linear activation functions (\eg,~ReLU) on the Lorentz representation will break the manifold constraint of the Lorentz model. 
We notice that the non-linear activation on the Poincar$\mathrm{\acute{e}}$ ball model $\mathcal{B}$ is manifold-preserving: for any $\Vec{b} \in \mathcal{B}$, we have $\sigma(\Vec{b}) \in \mathcal{B}$. 
Inspired by this, we project hyperbolic average $\Vec{m}_{i}^{\ell,\mathcal{L}}$ to the Poincar$\mathrm{\acute{e}}$ ball model to apply non-linear activation, and then project the result back to the Lorentz model, given by 
\begin{equation}
\label{equation:Lorentz_activation}
\Vec{h}_{i}^{\ell, \mathcal{L}} = p_{\mathcal{B} \to \mathcal{L}}\Big(\sigma\big(p_{\mathcal{L} \to \mathcal{B}}(\Vec{m}_{i}^{\ell, \mathcal{L}})\big)\Big), 
\end{equation}
where $p_{\mathcal{B} \to \mathcal{L}}(\cdot)$ and $p_{\mathcal{L} \to \mathcal{B}}(\cdot)$ are the isometric and isomorphic bijections between the Lorentz model and the Poincar$\mathrm{\acute{e}}$ ball model as defined in Eq.~\eqref{equation:bijection_LB}. 
After Eq.\eqref{equation:Lorentz_activation}, H2H-GCN obtains the output of the $\ell$-th layer: the $i$-th node's representation $\Vec{h}_{i}^{\ell, \mathcal{L}}$ on the Lorentz model.

\begin{table*}[!htb]  
\resizebox{\textwidth}{!}{
\centering
\begin{tabular}{@{}clcccccccc@{}}\cmidrule[\heavyrulewidth]{2-10}
& 
\textbf{Datasets} &\multicolumn{2}{c}{\textsc{Disease}}  &  \multicolumn{2}{c}{\textsc{Airport}}  &  \multicolumn{2}{c}{\textsc{PubMed}} & \multicolumn{2}{c}{\textsc{Cora}} \\
\cmidrule{3-10}
& \textbf{Methods} & {\textsc{LP}} & {\textsc{NC}} & {\textsc{LP}} & {\text{NC}} & {\textsc{LP}} & {\text{NC}} & {\textsc{LP}} & {\text{NC}}
\\ \cmidrule{2-10}
\multirow{4}{*}{\rotatebox{90}{}\rotatebox{90}{\hspace*{-6pt} Shallow}} 
& \textsc{Euc}~\cite{chami2019hyperbolic} & 59.8 $\pm$ 2.0 & 32.5 $\pm$ 1.1  &92.0 $\pm$ 0.0& 60.9 $\pm$ 3.4 & 83.3 $\pm$ 0.1 & 48.2 $\pm$ 0.7 & 82.5 $\pm$ 0.3 & 23.8 $\pm$ 0.7 \\ 
& \textsc{Hyp}~\cite{nickel2017poincare} & 63.5 $\pm$ 0.6 &  45.5 $\pm$ 3.3  & 94.5 $\pm$ 0.0& 70.2 $\pm$ 0.1 & 87.5 $\pm$ 0.1 & 68.5 $\pm$ 0.3 & 87.6 $\pm$ 0.2 & 22.0 $\pm$ 1.5  \\     
& \textsc{Euc-Mixed}~\cite{chami2019hyperbolic} &  49.6 $\pm$ 1.1 & 35.2 $\pm$ 3.4 &   91.5 $\pm$ 0.1 &68.3 $\pm$ 2.3 & 86.0 $\pm$ 1.3 & 63.0 $\pm$ 0.3& 84.4 $\pm$ 0.2 & 46.1 $\pm$ 0.4   \\  
& \textsc{Hyp-Mixed}~\cite{chami2019hyperbolic} & 55.1 $\pm$ 1.3 &  56.9 $\pm$ 1.5 & 93.3 $\pm$ 0.0& 69.6 $\pm$ 0.1 & 83.8 $\pm$ 0.3 & 73.9 $\pm$ 0.2& 85.6 $\pm$ 0.5 & 45.9 $\pm$ 0.3 \\  
\cmidrule{2-10}
\multirow{2}{*}{\rotatebox{90}{}\rotatebox{90}{NNs}}        
& \textsc{MLP}~\cite{chami2019hyperbolic} & 72.6 $\pm$ 0.6 & 28.8 $\pm$ 2.5  &  89.8 $\pm$ 0.5 & 68.6 $\pm$ 0.6 & 84.1 $\pm$ 0.9 & 72.4 $\pm$ 0.2 & 83.1 $\pm$ 0.5 & 51.5 $\pm$ 1.0 \\ 
& \textsc{HNN}~\cite{ganea2018hyperbolic} & 75.1 $\pm$ 0.3 & 41.0 $\pm$ 1.8   & 90.8 $\pm$ 0.2 & 80.5 $\pm$ 0.5 & 94.9 $\pm$ 0.1 & 69.8 $\pm$ 0.4 & 89.0 $\pm$ 0.1 & 54.6 $\pm$ 0.4\\
\cmidrule{2-10}
\multirow{3}{*}{\rotatebox{90}{}\rotatebox{90}{\hspace*{-14pt} GCNs}}
& \textsc{GCN}~\cite{kipf2016semi} & 64.7 $\pm $ 0.5 &  69.7 $\pm$ 0.4 &  89.3 $\pm$ 0.4 & 81.4 $\pm$ 0.6 & 91.1 $\pm$ 0.5 & 78.1 $\pm$ 0.2 & 90.4 $\pm$ 0.2 & 81.3 $\pm$ 0.3 \\ 
& \textsc{GAT}~\cite{velivckovic2017graph} & 69.8 $\pm $ 0.3  & 70.4 $\pm$ 0.4  &  90.5 $\pm$ 0.3 & 81.5 $\pm$ 0.3 & 91.2 $\pm$ 0.1 & 79.0 $\pm$ 0.3 & {93.7} $\pm$ 0.1 &  {83.0} $\pm$ 0.7\\ 
& \textsc{GraphSAGE}~\cite{hamilton2017inductive} & 65.9 $\pm$ 0.3  &  69.1 $\pm$ 0.6 &  90.4 $\pm$ 0.5 & 82.1 $\pm$ 0.5 & 86.2 $\pm$ 1.0 & 77.4 $\pm$ 2.2 & 85.5 $\pm$ 0.6 & 77.9 $\pm$ 2.4  \\ 
\multirow{3}{*}{ \rotatebox{90} {\hspace*{-24pt} \textsc{~Hyp} } \vspace{3pt} \rotatebox{90}{\hspace*{-24pt} GCNs }}  
& \textsc{SGC}~\cite{wu2019simplifying} & 65.1 $\pm$ 0.2  &  69.5 $\pm$ 0.2    & 89.8 $\pm$ 0.3 & 80.6 $\pm$ 0.1 & 94.1 $\pm$ 0.0 &  78.9 $\pm$ 0.0 & 91.5 $\pm$ 0.1 & 81.0 $\pm$ $0.1$ \\ 
\cmidrule{2-10}
& \textsc{HGCN}~\cite{chami2019hyperbolic}& {90.8} $\pm$ 0.3  & {74.5} $\pm$ 0.9 & {96.4} $\pm$ 0.1 & {90.6} $\pm$ 0.2 & {96.3} $\pm$ 0.0 & {80.3} $\pm$ 0.3 & 92.9 $\pm$ 0.1 & 79.9 $\pm$ 0.2 \\  
& H2H-GCN (Ours) & \textbf{97.8} $\pm$ 0.0& \textbf{90.3} $\pm$ 0.8  & \textbf{96.7} $\pm$ 0.0 & \textbf{91.0} $\pm$ 0.3 & \textbf{97.1} $\pm$ 0.0 & \textbf{82.3} $\pm$ 0.4 & \textbf{95.4} $\pm$ 0.0 & \textbf{83.6} $\pm$ 0.8  \\ 
\cmidrule[\heavyrulewidth]{2-10}
\end{tabular} }
\caption{ROC AUC for Link Prediction (LP), and F1 score (\textsc{Disease}, binary class) and accuracy (the others, multi-class) for Node Classification (NC) tasks. 
}
\label{table:hgcn_results}
\vspace{-2mm}
\end{table*}

\subsection{H2H-GCN Architecture}
\label{subsection:architecture}
\begin{small}
\begin{algorithm}[t]
\label{algorithm:H2H-GCN_embedding}
\small
\caption{\small{H2H-GCN embedding generation.}}
\KwIn{Graph $\mathcal{G}=(\mathcal{V}, \mathcal{E})$; node features $ \{\Vec{x}_{i}^{E}\}_{i \in \mathcal{V}}$; number of layers $L$; transformation matrices $\{\Mat{W}^{\ell}\}_{\ell=1}^{L}$; non-linearity activation function $\sigma(\cdot)$.}
\KwOut{H2H-GCN node embeddings $\{\Vec{h}_{i}^{L,\mathcal{L}}\}_{i \in \mathcal{V}}$.}
	Orthogonally initialize transformation sub-matrices $\{\widehat{\Mat{W}}^{\ell}\}_{\ell=1}^{L}$\;
	Construct $\Mat{W}^{\ell} = \begin{bmatrix}1 & \Vec{0}^{\top} \\ \Vec{0} & \widehat{\Mat{W}}^{\ell}\end{bmatrix},\forall \ell \in \{1,\cdots,L\}$\;
	Generate hyperbolic node representations $\{\Vec{x}_{i}^{\mathcal{L}}\}_{i \in \mathcal{V}}$ via Eq.~\eqref{equation:obtain_hyper}\;
	$\Vec{h}_{i}^{0, \mathcal{L}} = \Vec{x}_{i}^{\mathcal{L}}, \forall i \in \mathcal{V}$, \;
	\For {$\ell = 1$ to $L$}{
		Generate intermediate node representations $\{\Vec{\bar{h}}_{i}^{\ell, \mathcal{L}}\}_{i \in \mathcal{V}}$ via the hyperbolic feature transformation in Eq.~\eqref{equation:FA_in_L2}\;
		\For {$i \in \mathcal{V}$}{
			Generate hyperbolic average $\Vec{m}_{i}^{\ell, \mathcal{L}}$ by aggregating message from $\{\Vec{\bar{h}}_{j}^{\ell, \mathcal{L}}\}_{j \in \hat{\mathcal{N}}(i)}$ via the hyperbolic heighborhood aggregation in Eq.~\eqref{equation:einstein_midpoint}\;
			Generate the node representation $\Vec{{h}}_{i}^{\ell, \mathcal{L}}$ via the non-linear activation on $\Vec{m}_{i}^{\ell, \mathcal{L}}$ in Eq.~\eqref{equation:Lorentz_activation}\;
		}
	}
	\textbf{return} H2H-GCN node embeddings $\{\Vec{h}_{i}^{L,\mathcal{L}}\}_{i \in \mathcal{V}}$. 
\end{algorithm}
\end{small}

We summarize the H2H-GCN embedding generation algorithm as shown in Algorithm~\ref{algorithm:H2H-GCN_embedding}. 
Given a graph $\mathcal{G}=(\mathcal{V}, \mathcal{E})$ with a vertex set $\mathcal{V}$ and an edge set $\mathcal{E}$, H2H-GCN first maps input Euclidean node features $ \{\Vec{x}_{i}^{E} \}_{i \in \mathcal{V}}$ to hyperbolic space via Eq.~\eqref{equation:obtain_hyper}. 
The obtained hyperbolic node representations $ \{\Vec{h}_{i}^{0,\mathcal{L}} \}_{i \in \mathcal{V}}$ are sent to a multi-layer H2H-GCN. 
At the $\ell$-th layer, the input node representation $\Vec{h}_{i}^{\ell - 1, \mathcal{L}}$ from previous layer is passed through the hyperbolic feature transformation in Eq.~\eqref{equation:FA_in_L2}, and is updated via the hyperbolic neighborhood aggregation in Eq.~\eqref{equation:einstein_midpoint} and the non-linear activation in Eq.\eqref{equation:Lorentz_activation}. 
After $L$ layers, we obtain the H2H-GCN node embeddings $ \{\Vec{h}_{i}^{L,\mathcal{L}} \}_{i \in \mathcal{V}}$. 

For link prediction, we use the Fermi-Dirac decoder \cite{krioukov2010hyperbolic,nickel2017poincare} to calculate probability scores for the edge between the $i$-th node and the $j$-th node, given by  
\begin{equation}
\label{equation:fermi_decoder}
{p}\big((i,j) \in \mathcal{E} | \Vec{h}_{i}^{L,\mathcal{L}}, \Vec{h}_{j}^{L,\mathcal{L}}\big) = \big[e^{(d_{\mathcal{L}} {(\Vec{h}_{i}^{L,\mathcal{L}}, \Vec{h}_{j}^{L,\mathcal{L}})^{2}} -r)/t} + 1\big]^{-1}, 
\end{equation}
where $d_{\mathcal{L}}(\cdot,\cdot)$ is the hyperbolic distance function defined in Eq.~\eqref{equation:distance}, and $r$ and $t$ are hyper-parameters. 
Following \cite{chami2019hyperbolic}, we use the negative sampling strategy and the cross entropy loss for training. 

For node classification and graph classification, we exploit a centroid-based classification method studied in \cite{liu2019hyperbolic}. 
Specifically, we introduce a set of centroids $C=\{\Vec{c}_{1}^{\mathcal{L}}, \Vec{c}_{2}^{\mathcal{L}}, \cdots, \Vec{c}_{|C|}^{\mathcal{L}}\}$ lying on the Lorentz model, then calculate a distance matrix $\Mat{D} \in \mathbb{R}^{|\mathcal{V}| \times |C| }$ whose element $\Mat{D}_{i,j}=d_{\mathcal{L}}(\Vec{h}_{i}^{L,\mathcal{L}}, \Vec{c}_{j}^{\mathcal{L}})$ represents the distance between the $i$-th node embedding $\Vec{h}_{i}^{L,\mathcal{L}}$ and the $j$-th centroid $\Vec{c}_{j}^{\mathcal{L}}$. 
For node classification, we send $\Mat{D}_{i}$, the $i$-th row of distance matrix that contains the distance information between the $i$-th node and all centroids, to a classifier to predict the category of the $i$-th node. 
For graph classification, we apply average pooling to $\{\Mat{D}_{i}\}_{i=1}^{|\mathcal{V}|}$ as a readout operation to yield a graph embedding $\frac{1}{|\mathcal{V}|}\sum_{i=1}^{|\mathcal{V}|} \Mat{D}_{i}$, followed a classifier for prediction. 
On both classification tasks, we use softmax classifiers and cross entropy loss functions.

\vspace{-3pt}
\subsection{Optimization}
\label{subsection:optimization}
We explain how to learn parameters in H2H-GCN, especially the transformation matrix $\Mat{W}$ (omitting the layer number $^{\ell}$ for convenience) in the hyperbolic feature transformation Eq.~\eqref{equation:FA_in_L2}, that is an optimization problem with the orthogonal constraint. 
Other parameters can be learned by a standard gradient descent optimizer straightforwardly.

The transformation matrix $\Mat{W} = \begin{bmatrix}1 & \Vec{0}^{\top} \\\Vec{0} & \widehat{\Mat{W}}\end{bmatrix}$ is a block diagonal matrix that consists of a scalar $1$ and an orthogonal matrix $\widehat{\Mat{W}} \in \mathrm{St}(n',n)$ which resides on the Stiefel manifold. 
\vspace{-5pt}
\begin{definition}
\label{definition:Stiefel_manifold}
(The Stiefel manifold). 
The set of ($n' \times n$)-dimensional matrices, $n \le n'$, with orthonormal columns forms a compact Riemannian manifold called the Stiefel manifold $\mathrm{St(n',n)}$ \cite{boothby1986introduction}. 
\end{definition}
\begin{equation}
\label{equation:stiefel_manifold}
\mathrm{St(n',n)} \triangleq \{ \Mat{M} \in \mathbb{R}^{n' \times n} : \Mat{M}^{\top}\Mat{M} = \Mat{I} \}. 
\end{equation}

While updating $\Mat{W}$, we keep $1$ unchanged and introduce a Riemannian stochastic gradient descent optimizer to update $\widehat{\Mat{W}}$. 
Formally, let $J$ be the loss function, \eg,~cross entropy loss for classifications. 
$\widehat{\Mat{W}}$ is updated by 
\begin{equation}
\vspace{-5pt}
\label{equation:RSGD}
\left\{
   \begin{aligned}
   &  \Mat{P}^{(t)} = \eta \pi_{\widehat{\Mat{W}}^{(t)}}(\nabla_{\widehat{\Mat{W}}}^{(t)}) \\
   &  \widehat{\Mat{W}}^{(t+1)} = r_{\widehat{\Mat{W}}^{(t)}}(-\Mat{P}^{(t)}) \\
   \end{aligned}
\right., 
\vspace{0pt}
\end{equation}
where $\eta$ is the learning rate. 
$\nabla_{\widehat{\Mat{W}}}^{(t)} = \mathrm{d} J / \widehat{\Mat{W}}^{(t)}$ denotes the Eulcidean gradient of the loss function $J$ with respect to $\widehat{\Mat{W}}^{(t)}$ calculated at time $t$. 
$\pi_{{\widehat{\Mat{W}}^{(t)}}}(\cdot)$ is an orthogonal projection that transforms the Euclidean gradient to the Riemannian gradient 
\begin{equation}
\label{equation:ortho_proj}
\pi_{{\widehat{\Mat{W}}^{(t)}}} \big( {\nabla_{\widehat{\Mat{W}}}^{(t)}} \big) = {\nabla_{\widehat{\Mat{W}}}^{(t)}} - \frac{1}{2}{\widehat{\Mat{W}}^{(t)}} \Big( {\widehat{\Mat{W}}^{(t) \top}}{\nabla_{\widehat{\Mat{W}}}^{(t)}} + {\nabla_{\widehat{\Mat{W}}}^{(t)}}^{\top}{\widehat{\Mat{W}}^{(t)}} \Big). 
\end{equation}
$r_{{\widehat{\Mat{W}}^{(t)}}}(\cdot)$ is a retraction operation, defined as 
\begin{equation}
\label{equation:retraction}
\begin{aligned}
&r_{{\widehat{\Mat{W}}^{(t)}}}(-\Mat{P}^{(t)}) = \mathrm{{qf}}({\widehat{\Mat{W}}^{(t)}} - \Mat{P}^{(t)}), \\
\end{aligned}
\end{equation}
where $\mathrm{qf}(\cdot)$ extracts the orthogonal factor in the QR decomposition. 
The retraction operation prevents the updated $\widehat{\Mat{W}}^{(t+1)}$ from falling off the Stiefel manifold.

\section{Experiments}
\label{section:experiments}

We evaluate the proposed H2H-GCN on the link prediction, node classification and graph classification tasks, and comprehensively compare H2H-GCN with a variety of state-of-the-art Euclidean GCNs and hyperbolic GCNs.

\begin{table*}[!htbp]
\centering
\begin{tabular}{lccccc}
\toprule
& \multicolumn{5}{c}{\bf Dimensionality} \\
\cmidrule(l){2-6}
\bfseries Methods  & 3 & 5 & 10 & 20 & 256\\\midrule
Euclidean~\cite{liu2019hyperbolic} & 77.2 $\pm$ 0.12 & 90.0 $\pm$ 0.21 & 90.6 $\pm$ 0.17 & 94.8 $\pm$ 0.25 & 95.3 $\pm$ 0.17\\
HGNN~\cite{liu2019hyperbolic} & 94.1 $\pm$ 0.03 & 95.6 $\pm$ 0.14 & 96.4 $\pm$ 0.23 & 96.6 $\pm$ 0.22 & 95.3 $\pm$ 0.28\\
H2H-GCN (Ours) &  \textbf{95.4} $\pm$ 0.26  &  \textbf{96.7} $\pm$ 0.12  &  \textbf{96.8} $\pm$ 0.04  &  \textbf{97.0} $\pm$ 0.05  &  \textbf{97.2} $\pm$ 0.03 \\
\bottomrule
\end{tabular}
\vspace{5pt}
\caption{Results on synthetic graph classification where F1 (macro) score and standard deviation are reported. }
\label{table:synthetic}
\end{table*}
\subsection{Link Prediction and Node Classification}

The link prediction (LP) task is to predict the existence of links among nodes in a graph, and the node classification (NC) task is to predict labels of nodes in a graph. 
They have many applications such as predicting friendships among users in a social network and predicting research directions of papers in a citation network. 
For link prediction, we report area under the ROC curve (AUC), and for node classification, we report F1 score for binary-class datasets and accuracy for multi-class datasets. 
Following HGCN \cite{chami2019hyperbolic}, a link prediction regularization objective is added in the node classification task. 
\vspace{-5pt}

\paragraph{Datasets.}
\textsc{Disease} \cite{chami2019hyperbolic} is constructed by simulating the SIR disease spreading model \cite{anderson1992infectious}, where the label of a node indicates whether the node was infected or not, and the feature of a node indicates the susceptibility of the node to the disease. 
\textsc{Cora} and \textsc{Pubmed} \cite{sen2008collective} are citation network datasets where nodes are scientific papers and edges represent citation links. 
\textsc{Cora} contains $7$ classes of machine learning papers, and there are $2,708$ nodes, $5,429$ edges and $1,433$ features per node. 
\textsc{Pubmed} contains $3$ classes of medicine publications, and there are $19,717$ nodes, $44,338$ edges and $500$ features per node. 
\textsc{Airport} \cite{chami2019hyperbolic} is a flight network dataset where nodes are airports and edges represent the airline routes. 
There are $2,236$ nodes in total and the label of a node is the population of the country where the airport (node) belongs to. 
\vspace{-12pt}
\paragraph{Baselines.}
We consider four types of baseline methods: shallow methods, neural networks (NNs) methods, GCNs, and hyperbolic GCNs (\textsc{Hyp GCNs}). 
Shallow methods optimize to minimize a reconstruction loss, and the parameters in models act as an embedding look-up table. 
We consider Euclidean embeddings (\textsc{Euc}) \cite{chami2019hyperbolic} and its hyperbolic extension (\textsc{Hyp}) \cite{nickel2017poincare}. 
As the two embeddings fail to leverage node features, \textsc{Euc-Mixed} \cite{chami2019hyperbolic} and \textsc{Hyp-Mixed} \cite{chami2019hyperbolic} concatenate the shallow embeddings with node features for a fair comparison with other methods using node features. 
NNs methods only utilize the node features but does not consider graph structures. 
Compared NNs methods include Euclidean multi-layer percerptron (MLP) and its hyperbolic extension: hyperbolic neural networks (HNN)~\cite{ganea2018hyperbolic}. 
For GCNs, we compare H2H-GCN with several Euclidean state-of-the-art GCNs models: GCN \cite{kipf2016semi}, \textsc{GraphSAGE} \cite{hamilton2017inductive}, graph attention networks (GAT) \cite{velivckovic2017graph}, and simplified graph convolution (SGC) \cite{wu2019simplifying}. 
For \textsc{Hyp GCNs}, we consider HGCN \cite{chami2019hyperbolic} that performs graph convolutional operations in tangent spaces.

The comparisons are presented in Table~\ref{table:hgcn_results}. 
We notice that HNN, a generalization of MLP to hyperbolic spaces, outperforms MLP on most tasks. 
It indicates that hyperbolic spaces are more suitable for modeling graphs compared than Euclidean spaces. 
A similar conclusion can be drawn while comparing Euclidean GNNs with hyperbolic GCNs. 
HGCN works better than state-of-the-art Euclidean GCNs on most datasets. 

The proposed H2H-GCN consistently achieve the best performance on both tasks on all datasets. 
We take the \textsc{Disease} dataset as an example to analyze the effectiveness of H2H-GCN. 
The \textsc{Disease} dataset is a tree network that possesses a strong hierarchical structure. 
As hyperbolic spaces can be viewed as smooth versions of trees, hyperbolic GCNs should work better than Euclidean GCNs. 
The results are in line with expectation that both HGCN and H2H-GCN show significant improvement on the LP and NC tasks compared with Euclidean methods. 
In particular, H2H-GCN achieves an average of $22.7\%$ (LP) and $19.9\%$ (NC) performance gains than state-of-the-art Euclidean GCNs, and $7.0\%$ (LP) and $15.8\%$ (NC) performance gains than HGCN. 
It demonstrates that our H2H-GCN is superior to hyperbolic GCNs which rely on tangent spaces. 
We owe it to our developed hyperbolic graph convolution that directly works on the hyperbolic manifold. 
Both the hyperbolic feature transformation and the hyperbolic neighborhood aggregation are manifold-preserving. 
It can avoid the distortion caused by tangent space approximations and keep the global hyperbolic geometry underlying graphs.

\subsection{Graph Classification}

Graph classification is important to many real-world applications associated with graph-structured data, \eg, chemical drug analysis. 
We first evaluate our method on a synthetic graph dataset, and then on several commonly-used molecular graph datasets. 
\vspace{-8pt}

\subsubsection{Synthetic Graphs}
Following \cite{liu2019hyperbolic}, we take three graph generation algorithms: Erd\H{o}s-R\'enyi \cite{Erdos:1959random}, Barab\'asi-Albert \cite{Barabasi:1999emergence} and Watts-Strogatz \cite{Watts:1998smallworld} to construct a synthetic graph dataset. 
For each graph generation algorithm, we generate $6,000$ graphs and divide them into 3 equal parts for the training, validation, and test (see \cite{liu2019hyperbolic} for more generation details). 
Typical properties, such as small-world property of graphs generated by Watts-Strogatz algorithm and scale-free property of graphs generated by Barab\'asi-Albert algorithm, can be explained by an underlying hyperbolic geometry \cite{krioukov2010hyperbolic}, thus it is more suitable for modeling such graphs in hyperbolic spaces than in Euclidean spaces. 

We compare the proposed H2H-GCN with Euclidean embeddings \cite{liu2019hyperbolic} and HGNN \cite{liu2019hyperbolic}. 
The F1 scores of different embedding dimensionalities are presented in Table~\ref{table:synthetic}. 
The performance of HGNN and our method surpasses Euclidean embeddings by a large margin when embedding dimensionality is low. 
It is because hyperbolic spaces can well capture the hyperbolic geometry underlying these synthetic graphs. 
As the embedding dimensionality increases (\eg,~$256$), HGNN tends to be comparable with Euclidean alternative $95.3\%$, while our method achieves $97.2\%$, $1.9\%$ higher than HGNN. 
The reason is that the distortion caused by tangent space approximations in HGNN becomes signficant with the increase of dimensionality of embedding spaces, leading to an inferior performance. 
H2H-GCN tackles this problem by directly learning node embeddings in the hyperbolic space. 
It shows the best performance from $3$-dimensional embeddings to $256$-dimensional embeddings.

\vspace{-8pt}
\subsubsection{Molecular Graphs}
We evaluate our method on several chemical datasets to predict the function of molecular graphs. 
\textsc{D\&D} \cite{dobson2003distinguishing} consists of X-ray crystal structures. 
It has $1,178$ graphs in total, and $2$ classes indicting the molecular is an enzyme or not. 
\textsc{Proteins} \cite{borgwardt2005protein} has $1,113$ graphs, and $3$ classes of graphs representing helix, sheet or turn. 
\textsc{Enzymes} \cite{schomburg2004brenda} contains $600$ graphs and $6$ classes in total.

We compare our method with several state-of-the-art Euclidean GCNs, including DGCNN \cite{zhang2018end}, \textsc{DiffPool} \cite{ying2018hierarchical}, ECC \cite{simonovsky2017dynamic}, GIN \cite{xu2018powerful} and \textsc{GraphSAGE} \cite{hamilton2017inductive}, and a hyperbolic GCN, \ie,~ HGNN \cite{liu2019hyperbolic} that performs Euclidean graph convolutional operations in tangent spaces. 
In general, researchers adopt tenfold cross validation for model evaluation. 
However, as pointed out by \cite{errica2019fair}, the data splits are different and the experimental procedures are often ambiguous of different works, which results in unfair comparisons. 
To this end, the work in \cite{errica2019fair} provides a uniform and rigorous benchmarking of state-of-the-art models. 
In this part, we follow the same experimental procedures and use the same data splits as \cite{errica2019fair} for fair comparisons.

\begin{table}[!htb]
\resizebox{0.482\textwidth}{!}{
\centering
\begin{tabular}{l  c c c}
\midrule[\heavyrulewidth]
\bfseries Methods     &  \textsc{D\&D}  &  \textsc{Proteins}  & \textsc{Enzymes}  \\
\hline
\textsc{DGCNN}~\cite{zhang2018end}              & 76.6 $\pm$ 4.3  & 72.9 $\pm$ 3.5 & 38.9 $\pm$ 5.7 \\
\textsc{DiffPool}~\cite{ying2018hierarchical}   & 75.0 $\pm$ 3.5  & 73.7 $\pm$ 3.5 & 59.5 $\pm$ 5.6 \\
\textsc{ECC}~\cite{simonovsky2017dynamic}       & 72.6 $\pm$ 4.1  & 72.3 $\pm$ 3.4 & 29.5 $\pm$ 8.2 \\
\textsc{GIN}~\cite{xu2018powerful}              & 75.3 $\pm$ 2.9  & 73.3 $\pm$ 4.0 & 59.6 $\pm$ 4.5 \\
\textsc{GraphSAGE}~\cite{hamilton2017inductive} & 72.9 $\pm$ 2.0  & 73.0 $\pm$ 4.5 & 58.2 $\pm$ 6.0 \\
\hline
\textsc{HGNN}~\cite{liu2019hyperbolic}          & 75.8 $\pm$ 3.3  & 73.7 $\pm$ 2.3 & $51.3 \pm 6.1$ \\
\textsc{H2H-GCN} (Ours)            & \textbf{78.2} $\pm$ 3.3  & \textbf{74.4} $\pm$ 3.0 & \textbf{61.3} $\pm$ 4.9\\
\midrule[\heavyrulewidth]
\end{tabular}
}
\caption{Results on chemical graph classification where mean accuracy and standard deviation are reported. }
\label{table:result_summary}
\end{table}

The mean accuracy and standard deviation are reported in Table~\ref{table:result_summary}. 
We observe that HGNN are comparable with Euclidean methods, which illustrates two possible reasons: either hyperbolic GCNs are not suitable for the three datasets, or some factors in HGNN limit representation ability of hyperbolic GCNs. 
The performance of H2H-GCN may give the answer. 
It achieves the best performance on all datasets: $1.6\%$ higher than \textsc{DGCNN} on \textsc{D\&D}, $0.7\%$ higher than \textsc{DiffPool} on \textsc{Proteins}, and $1.7\%$ higher than \textsc{GIN} on \textsc{Enzymes}. 
It indicates that there are some limitations in HGNN and H2H-GCN overcomes it. 
Compared with HGNN that does graph convolutional operations in the tangent space, the key difference lies in the propoed H2H-GCN performs a hyperbolic graph convolution in the hyperbolic space. 
In this way, H2H-GCN preservs the global hyperbolic geometry, leading to a superior performance.

\begin{table}[!tbp]
\scriptsize
\renewcommand\tabcolsep{2.8pt} 
\centering
\begin{tabular}{lccccc}
\toprule
& \multicolumn{5}{c}{\bf Dimensionality} \\
\cmidrule(l){2-6}
\bfseries Methods & 3 & 5 & 10  & 20 & 128 \\\midrule
Euclidean~\cite{liu2019hyperbolic}	& 64.2 $\pm$ 4.9  &  71.2 $\pm$ 3.4  &  76.2 $\pm$ 1.5  &  78.1 $\pm$ 2.1 &  80.4 $\pm$ 0.9 \\
HGNN~\cite{liu2019hyperbolic}		& 65.3 $\pm$ 3.6  &  71.0 $\pm$ 3.4  &  76.1 $\pm$ 1.5  &  79.2 $\pm$ 1.6 &  80.1 $\pm$ 0.9 \\
HGCN~\cite{chami2019hyperbolic}		& 70.8 $\pm$ 1.6  &  75.4 $\pm$ 1.7  &  78.1 $\pm$ 0.8  &  79.7 $\pm$ 1.4 &  81.7 $\pm$ 0.7 \\
H2H-GCN (Ours)	& \textbf{73.1} $\pm$ 2.5  &  \textbf{77.8} $\pm$ 0.6  &  \textbf{79.9} $\pm$ 0.9  &  \textbf{81.2} $\pm$ 0.9 &  \textbf{83.6} $\pm$ 0.8  \\
\bottomrule
\end{tabular}
\vspace{5pt}
\caption{Comparisons of embedding dimensionality for node classification on \textsc{Cora} where accuracy and standard deviation are reported. }
\label{table:cora_dim}
\vspace{-12pt}
\end{table}


\subsection{Dimensionality Comparisons}
We test the effect of embedding dimensionality from $3$ to $128$ for node classification on \textsc{Cora}, and report the performance of Euclidean embedding \cite{liu2019hyperbolic}, HGNN \cite{liu2019hyperbolic}, HGCN \cite{chami2019hyperbolic} and the proposed H2H-GCN in Table~\ref{table:cora_dim}. 
HGNN gets comparable results with Euclidean embeddings, while HGCN shows pretty improvments. 
It outperforms HGCN, $2.3\%$ and $1.9\%$ higher than HGCN when embedding dimensionalities are $3$ and $256$, respectively. 
We claim that the distortion caused by tangent space approximations exists in both low and high embedding dimensionalities. 
Although increasing the dimensionality can improve performance, it cannot solve this problem. 
H2H-GCN tackles it by proposing a hyperbolic graph convolution to directly work on the hyperbolic manifold. 
Such a manifold-to-manifold method achieves remarkable improvments. 
H2H-GCN works the best in all embedding dimensionalities.

\section{Conclusion}
\label{section:conclusion}

In this paper, we have presented a hyperbolic-to-hyperbolic graph convolutional network (H2H-GCN) for embedding graph with hierarchical structure into hyperbolic spaces. 
The developed hyperbolic graph convolution which consists of a hyperbolic feature transformation and a hyperbolic neighborhood aggregation, can be directly conducted on hyperbolic manifolds. 
The both operations can ensure that the output still lies on the hyperbolic manifold. 
In constrast to existing hyperbolic GCNs relying on tangent spaces, H2H-GCN can avoid the distortion caused by tangent space approximations and keep the global hyperbolic geometry underlying graphs. 
Extensive experiments on link prediction, node classification and graph classification have showed that H2H-GCN achieves competitive results compared with state-of-the-art Euclidean GCNs and existing hyperbolic GCNs. 

\vspace{-5pt}
\paragraph{Acknowledgments.}
This work was supported by the Natural Science Foundation of China (NSFC) under Grants No.  61773062 and No. 62072041.

{\small
\bibliographystyle{ieee_fullname}
\bibliography{egbib}
}

\end{document}